\documentclass[letterpaper]{article} % DO NOT CHANGE THIS
\pdfoutput=1
\usepackage{aaai24}  % DO NOT CHANGE THIS
\usepackage{times}  % DO NOT CHANGE THIS
\usepackage{helvet}  % DO NOT CHANGE THIS
\usepackage{courier}  % DO NOT CHANGE THIS
\usepackage[hyphens]{url}  % DO NOT CHANGE THIS
\usepackage{graphicx} % DO NOT CHANGE THIS
\usepackage{natbib}  % DO NOT CHANGE THIS AND DO NOT ADD ANY OPTIONS TO IT
\usepackage{caption} % DO NOT CHANGE THIS AND DO NOT ADD ANY OPTIONS TO IT
\usepackage[ruled,linesnumbered]{algorithm2e}
\usepackage{amsthm,amsmath,amssymb,amsfonts}
\usepackage{subfigure}
\usepackage{mdframed}
\usepackage{multirow}
\usepackage{booktabs}

\newtheorem{proposition}{Proposition}

\usepackage{newfloat}
\usepackage{listings}
\DeclareCaptionStyle{ruled}{labelfont=normalfont,labelsep=colon,strut=off} % DO NOT CHANGE THIS
\title{Upper Bounding Barlow Twins: A Novel Filter for Multi-Relational Clustering}
\author{
    Xiaowei Qian\textsuperscript{\rm 1}\equalcontrib,
    Bingheng Li\textsuperscript{\rm 1}\equalcontrib,
    Zhao Kang\textsuperscript{\rm 1}\footnote{Corresponding author}
}
\affiliations{
    \textsuperscript{\rm 1} University of Electronic Science and Technology of China,
    Chengdu, Sichuan, China\\
    \{xiaoweiqian0311, bingheng86\}@gmail.com, zkang@uestc.edu.cn
}

\begin{document}

\maketitle

\begin{abstract}
Multi-relational clustering is a challenging task due to the fact that diverse semantic information conveyed in multi-layer graphs is difficult to extract and fuse. Recent methods integrate topology structure and node attribute information through graph filtering. However, they often use a low-pass filter without fully considering the correlation among multiple graphs. To overcome this drawback, we propose to learn a graph filter motivated by the theoretical analysis of Barlow Twins. We find that input with a negative semi-definite inner product provides a lower bound for Barlow Twins loss, which prevents it from reaching a better solution. We thus learn a filter that yields an upper bound for Barlow Twins. Afterward, we design a simple clustering architecture and demonstrate its state-of-the-art performance on four benchmark datasets. The source code is available at \textit{https://github.com/XweiQ/BTGF.}
\end{abstract}

\section{Introduction}
The advancements in information technology have led to a substantial proliferation of complex data, e.g., non-Euclidean graphs and multi-view data. Data originating from a variety of sources, each of which exhibits different characteristics, are often referred to as multi-view data. As a special type of multi-view data, multi-relational graphs contain two or more relations over a vertex set \cite{qu2017attention}. For instance, in the case of social networks, users and their profiles are considered as nodes and attributes, where each user interacts with others through multiple types of relationships such as friendship, colleague, and co-following.

Clustering is a practical technique to handle rich multi-relational graphs by finding a unique cluster pattern of nodes. One principle underlying multi-relational clustering is to leverage consistency and complementarity among multiple views to achieve good performance. For example, SwMC \cite{nie2017self} learns a shared graph from multiple graphs by using a weighting strategy; O2MAC \cite{fan2020one2multi} extracts shared representations across multiple views from the most informative graph; MCGC \cite{pan2021multi} utilizes a set of adaptive weights to learn a high-quality graph from the original multi-relational graphs. A key component of these methods is graph filtering, which fuses the topology structure and attribute information. They show that impressive performance can be achieved even without using neural networks \cite{lin2023multi, pan2023high}. This provides a smart way for traditional machine learning methods to benefit from representation learning techniques. Nevertheless, they simply use a low-pass filter without fully considering the correlation between different views. Moreover, these filters are empirically designed and fixed, which is not flexible to suit different data.

How to explore the correlation among multiple graphs is a critical problem in multi-view learning. Lyu \textit{et al} \cite{lyu2022understanding} theoretically illustrate that the correlation-based objective functions are effective in extracting shared and private information in multi-view data under some assumptions. Among them, Barlow Twins \cite{zbontar2021barlow} is particularly popular. It consists of two parts: the invariance term maximizes the correlation between the same feature across different views, while the redundancy term decorrelates different features across various views. The feature decorrelation operation not only exploits the correlation of multiple views but also effectively alleviates the problem of representation collapse in self-supervised learning. This idea has been applied to graph clustering, such as MVGC \cite{xia2022multi} and MGDCR \cite{mo2023multiplex}. However, existing methods simply use Barlow Twins, without any special operations catering to multi-relational graphs. Consequently, they still suffer from collapse. To show this, we visualize the feature distributions of several representative methods in ACM data: contrastive learning-based method MGCCN \cite{liu2022deep}, Barlow Twins-based method MGDCR \cite{mo2023multiplex}, and our proposed method \textbf{B}arlow \textbf{T}wins \textbf{G}uided \textbf{F}ilter (BTGF). Comparing Figs. \ref{collapse:MGCCN} and \ref{collapse:MGDCR}, we can observe the advantage of Barlow Twins. From Figs. \ref{collapse:MGDCR} and \ref{collapse:BTFG}, a more evident enhancement in BTGF can be found. 

In this work, we reveal that an input with a negative semi-definite inner product will lead to a lower bound for Barlow Twins loss, while an input with a positive semi-definite inner product has an upper bound. To minimize Barlow Twins loss as much as possible, we employ a graph filter to make the inner product positive semi-definite. Therefore, our filter upper bounds Barlow Twins, which means that the loss will never be too large to dominate other terms. 

\begin{figure*}[htb]
\centering
\subfigure[MGCCN]{
\includegraphics[width=5.5cm]{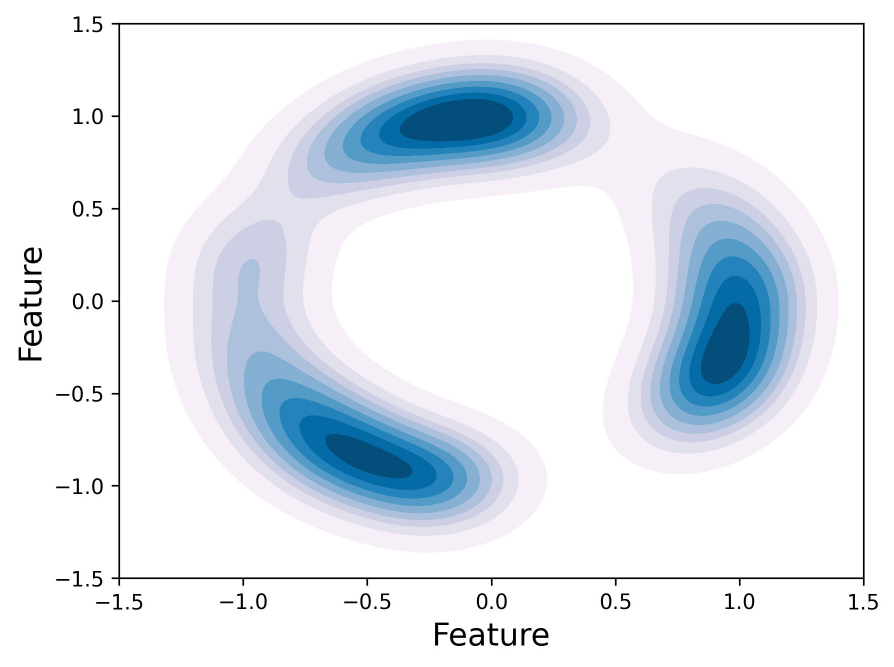}\label{collapse:MGCCN}
\hspace{0mm}
}
\subfigure[MGDCR]{
\includegraphics[width=5.5cm]{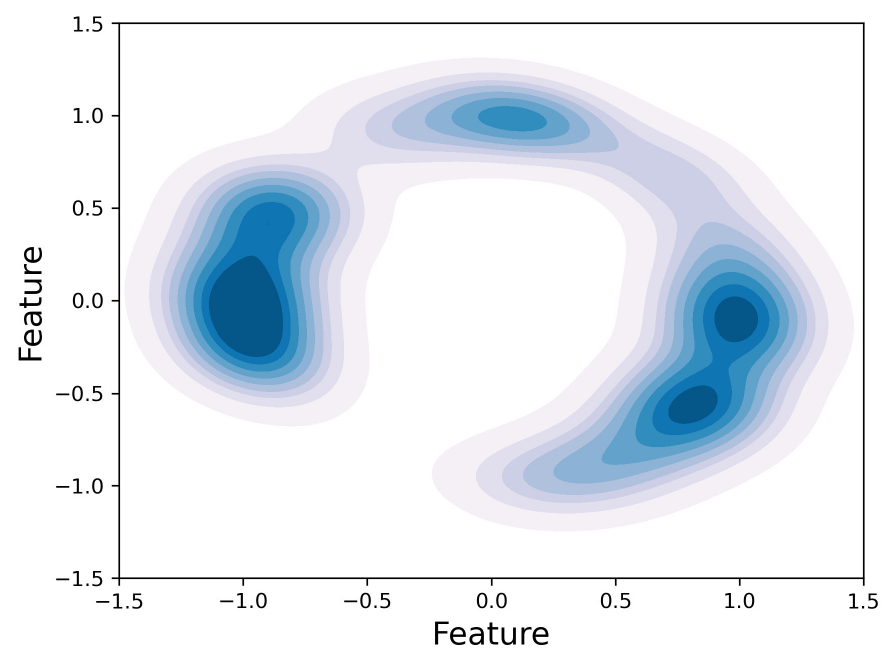}\label{collapse:MGDCR}
\hspace{0mm}
}
\subfigure[BTGF]{
\includegraphics[width=5.5cm]{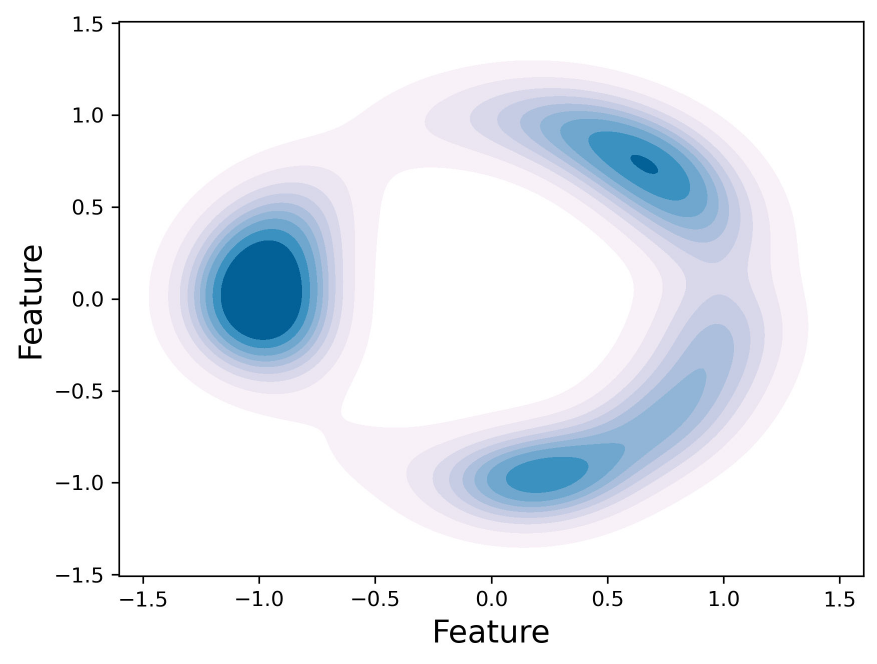}\label{collapse:BTFG}
\hspace{0mm}
}
\caption{Visualization of representation distributions utilizing Gaussian Kernel Density Estimation (KDE) \cite{botev2010kernel} on ACM. Representations are mapped onto a two-dimensional normalized vector through t-SNE. A darker color indicates a higher concentration of points within the area. %\cite{yu2022graph}
A better collapse mitigation method makes the feature distributions more uniform.}
\label{collapse}
\end{figure*}

The overall architecture for multi-relational clustering is displayed in Fig. \ref{fig:architecture}. The learned graph filter aggregates information for each node from its neighbors, resulting in smooth representations. They serve as input to an encoder, which maps them into a clustering-favorable space. Subsequently, they are reconstructed by a decoder. After training, clustering results are obtained using the output of the encoder.

The main contributions of this paper can be summarized in three aspects:
\begin{itemize}
    \item We conduct a theoretical analysis to examine how the input affects the optimization of Barlow Twins. An input with a negative semi-definite inner product provides a lower bound for Barlow Twins loss, while an input with a positive semi-definite inner product yields an upper bound.
    
    \item A graph filter that facilitates Barlow Twins optimization is designed.  This filter ensures that the inner product of the encoder input is positive semi-definite, which enables Barlow Twins to have an upper bound and surpass the lower bound, leading to improved performance.

    \item A simple yet effective clustering architecture is developed. Experimental results on four multi-relational graph datasets demonstrate the superiority of BTGF, even when the network just employs a linear layer.  
\end{itemize}

\section{Related Work}

In recent years, numerous multi-view graph clustering methods have been proposed. Shallow methods MvAGC \cite{lin2021graph} and MCGC \cite{pan2021multi} employ a low-pass filter to embed relation information into attributes, and have achieved impressive results. Nevertheless, they just use a simple weight to differentiate various views and don't explicitly consider the correlations among different views.

Distinct from the shallow methods described above, deep methods attempt to learn good representations via designed neural networks. O2MAC \cite{fan2020one2multi} and MAGCN \cite{cheng2021multi} cluster multi-relational graphs using GCN. CMGEC \cite{wang2021consistent} applies mutual information maximization to capture complementary and consistent information of each view. HAN \cite{10.1145/3308558.3313562} and HDMI \cite{jing2021hdmi} apply the attention mechanism to fuse different relations. Due to the importance of structure in different views \cite{fang2022structure}, MGCCN \cite{10.1016/j.ins.2022.09.042} introduces a contrastive learning mechanism to capture consistent information between diverse views. However, these methods could be subject to representation collapse.

Two popular solutions to mitigate representation collapse are asymmetric model architecture, e.g., MoCo \cite{he2020momentum}, BYOL \cite{grill2020bootstrap}, SimSiam \cite{chen2021exploring}, and appropriate objective function, e.g., SimCLR \cite{chen2020simple}, Barlow Twins \cite{zbontar2021barlow}. BT introduces a novel cross-correlation objective function for feature decorrelation. Owing to its concise implementation, without negative sample generation and asymmetric networks, it has gained popularity in self-supervised learning \cite{zhang2021zero, zhang2022align}. Graph Barlow Twins \cite{bielak2022graph} optimizes the embeddings of two distorted views of a graph. Clustering methods are also prone to collapse, e.g., empty clusters in K-means. DCRN \cite{liu2022deep} extends the idea of Barlow Twins into deep clustering and designs a dual correlation reduction network to address representation collapse. However, it is a single-view model and cannot handle multi-relational graphs. DGCN \cite{pmlr-v202-pan23b} uses similar correlation reduction item to alleviate collapse. MVGC \cite{xia2022multi} and MGDCR \cite{mo2023multiplex} directly apply Barlow Twins to multi-relational graphs. However, they suffer insufficient optimization.
In this paper, we theoretically analyze the conditions for the existence of lower and upper bounds for Barlow Twins loss and design a new filter based on it.

\begin{figure*}[htb]
    \centering
    \includegraphics[width=1\textwidth]{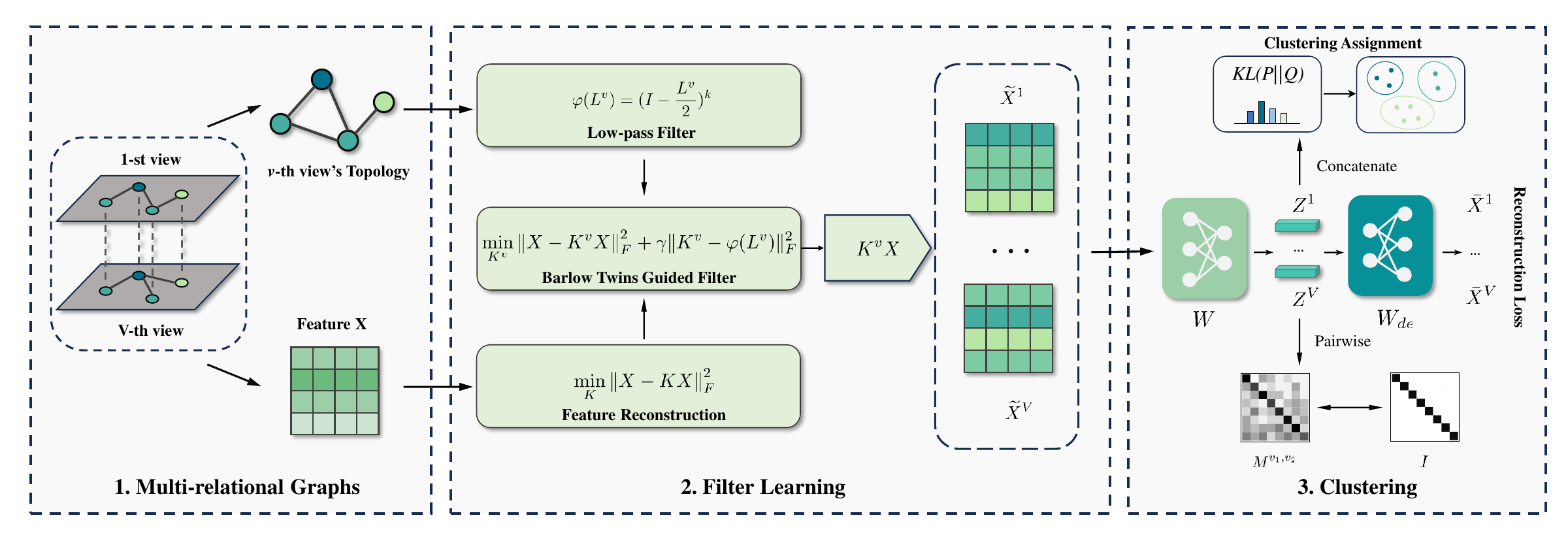}
    \caption{The proposed framework for multi-relational clustering.}
    \label{fig:architecture}
\end{figure*}

\section{Methodology}

\subsection{Notation}

Define the multi-relational graph as $G=\left\{\mathcal{V}, E_1, \ldots, E_{V}, X\right\}$, where $\mathcal{V}$ represents the sets of $n$ nodes, $e_{i j} \in E_v$ denotes the relationship between node $i$ and node $j$ in the $v$-th view. ${V} \geq 1$ is the number of relational graphs. $X=\left\{x_1, \ldots, x_n\right\}^{\top} \in {R}^{n \times f}$ is the attribute matrix, $f$ is the dimension of features. The adjacency matrix $\widetilde{A}^v$ characterizes the initial graph structure of the $v$-th view. $D^v$ represents the degree matrices. The normalized adjacency matrix is $A^v=\left(D^v\right)^{-\frac{1}{2}}\left(\widetilde{A}^v+I\right)\left(D^v\right)^{-\frac{1}{2}}$ and the corresponding graph Laplacian is $L^v=I-A^v$. And $c$ represents the number of node classes.

\subsection{Barlow Twins Guided Filter}

To inject the topology structure into features, graph filtering is performed as:
\begin{equation}
    \widetilde{X}^{v} = g(L^{v}) X,
\end{equation}
where $g(L^v)$ is the $v$-th view's graph filter.
Then, we represent the encoder as a multi-layer dimensional transformation $W \in R^{f\times d}$, where $d$ is the output dimension of the embedding $Z$. We can express $Z^{v}$ as:
\begin{equation}\label{embedding}
    Z^{v} = g(L^{v}) X W.
\end{equation}
Specifically, we aim to learn a single encoder $h: R^{n\times f} \to R^{n\times d}$ for different views, i.e., the encoders $Z^v = h(\widetilde{X}^v)$  share the same parameters.
Afterward, we compute Barlow Twins (BT) as follows:
\begin{equation}
\begin{aligned}\label{barlowtwins1}
    \mathcal{L}^{v_1,v_2}_{BT} &=\sum_{i=1}^{d}\left(M^{v_1,v_2}_{i i}-1\right)^{2}+\lambda \sum_{i=1}^{d} \sum_{j \neq i}^d (M^{v_1,v_2}_{i j})^{2}, \\
    % &M^{v_1,v_2}_{i j} = \frac{{Z^{v_1}_{:i}}^{\top} Z^{v_2}_{:j}}{\left\|Z^{v_1}_{:i}\right\| \cdot\left\|Z^{v_2}_{:j}\right\|} = \hat{Z}^{v_1}_{:i}^{\top} \hat{Z}^{v_2}_{:j},
    &M^{v_1,v_2}_{i j} = \frac{{Z^{v_1}_{:i}}^{\top} Z^{v_2}_{:j}}{\left\|Z^{v_1}_{:i}\right\| \cdot\left\|Z^{v_2}_{:j}\right\|} = (\hat{Z}^{v_1}_{:i})^{\top} \hat{Z}^{v_2}_{:j},
\end{aligned}
\end{equation}  
where  $\hat{Z}^{v}$ is the column-normalized form of $Z^{v}$, i.e., $\hat{Z}^{v} = Z^{v}(\Lambda^{v})^{-1}$, $\Lambda^{v}$ is a diagonal matrix with positive elements $\left\|Z^{v}_{:i}\right\|$. $\lambda$ is a trade-off parameter that balances the invariance term and redundancy reduction term, which is fixed to 0.0051 \cite{zbontar2021barlow}.

To enhance the effectiveness of BT loss on
multi-relational graphs, we first provide some theoretical analysis. We find that the convergence of BT is influenced by the inner product between $\widetilde{X}^{v_1}$ and $\widetilde{X}^{v_2}$.

\begin{proposition}\label{proposition1} 
Barlow Twins has an explicit lower bound $\sum_{i=1}^d (\frac{\Lambda^{v_1}_{ii} \Lambda^{v_2}_{ii}}{max(\Lambda^{v_1}_{ii}) max(\Lambda^{v_2}_{ii})})^2$, if $(g(L^{v_{1}})X)^{\top} g(L^{v_{2}}) X$ is negative semi-definite. 
\end{proposition}

\begin{proof}
Denote $(g(L^{v_1})X)^{\top} g(L^{v_2}) X $ as $H^{v_1, v_2} \in R^{f \times f}.$ 
\begin{equation}\label{proof1}
\begin{aligned}
    M^{v_1, v_2}
    &= (\Lambda^{v_1})^{-1} {W}^{\top} (g(L^{v_1})X)^{\top} g(L^{v_2}) X W (\Lambda^{v_2})^{-1} \\
    &= (\Lambda^{v_1})^{-1}{W}^{\top} H^{v_1, v_2} W (\Lambda^{v_2})^{-1},
    \\ 
    \mathcal{L}^{v_1,v_2}_{BT} 
    &=\sum_{i=1}^{d}\left(M^{v_1,v_2}_{i i}-1\right)^{2}+\lambda \sum_{i=1}^{d} \sum_{j \neq i}^d (M^{v_1,v_2}_{i j})^{2}\\
    & \underset{(1.a)}{>} \sum_{i=1}^{d}\left(M^{v_1, v_2}_{i i}-1\right)^{2}\\
    &= \sum_{i=1}^{d} (\frac{1}{\Lambda^{v_1}_{ii} \Lambda^{v_2}_{ii}} {W}_{:i}^{\top} H^{v_1, v_2} {W}_{:i}-1)^{2}\\
    & \geq \sum_{i=1}^{d} (\frac{{W}_{:i}^{\top} H^{v_1, v_2} {W}_{:i}-\Lambda^{v_1}_{ii} \Lambda^{v_2}_{ii}}{max(\Lambda^{v_1}_{ii}) max(\Lambda^{v_2}_{ii})})^{2}\\
    & \underset{(1.b)}{\geq} \sum_{i=1}^{d} (\frac{\Lambda^{v_1}_{ii} \Lambda^{v_2}_{ii}}{max(\Lambda^{v_1}_{ii}) max(\Lambda^{v_2}_{ii})})^{2},\\
\end{aligned}
\end{equation}
where (1.a): $\lambda > 0$ and 
(1.b): Since $H^{v_1, v_2}$ is negative semi-definite, we have 
${W}_{:i}^{\top} H^{v_1, v_2} {W}_{:i} \leq 0$, which implies $({W}_{:i}^{\top} H^{v_1, v_2} {W}_{:i}-\Lambda^{v_1}_{ii} \Lambda^{v_2}_{ii})^{2} \geq (\Lambda^{v_1}_{ii} \Lambda^{v_2}_{ii})^{2}.$
\end{proof}

The optimization objective of BT is to make the diagonal elements converge to 1 and the off-diagonal elements converge to zero, thus driving the BT value close to zero. However, having a constant positive lower bound in Proposition. \ref{proposition1} undermines the capability of BT to effectively explore multi-view correlations and alleviate representation collapse. Additionally, clustering models often involve several loss terms for multi-objective optimization. Consequently, BT could be influenced by other losses during the training process, which makes it impossible to reach zero. To address this problem, we attempt to impose an upper bound on BT.

\begin{proposition}\label{remark1}
If $H^{v_1, v_2}$ is a positive semi-definite matrix,  
Barlow Twins will be upper bound by $d+\lambda \sum_{i=1}^{d} \sum_{j \neq i}^d (M^{v_1, v_2}_{i j })^{2}$. 
\end{proposition}

\begin{proof}
\begin{equation}
    \begin{aligned}
    \mathcal{L}^{v_1,v_2}_{BT} 
    &= \sum_{i=1}^{d} (\frac{1}{\Lambda^{v_1}_{ii} \Lambda^{v_2}_{ii}} {W}_{:i}^{\top} H^{v_1, v_2} {W}_{:i}-1)^{2} \\
    &+\lambda \sum_{i=1}^{d} \sum_{j \neq i}^d (M^{v_1, v_2}_{i j})^{2}\\
    &\underset{(2.a)}{\leq} d+\lambda \sum_{i=1}^{d} \sum_{j \neq i}^d (M^{v_1, v_2}_{i j })^{2},
\end{aligned}
\end{equation}
where (2.a): $M^{v_1, v_2}_{i i}$ is the cosine similarity of $Z^{v_{1}}_{: i}$ and $Z^{v_{2}}_{: i}$, so $M^{v_1, v_2}_{ i i }\in [-1,1]$. If $H$ is  positive semi-definite, we have $M^{v_1, v_2}_{i i}=\frac{1}{\Lambda^{v_1}_{ii} \Lambda^{v_2}_{ii}} {W}_{:i}^{\top} H^{v_1, v_2} {W}_{:i} \geq 0 $, i.e., $M^{v_1, v_2}_{ i i }\in [0,1]$. Therefore, the first term is bounded by $d$.
\end{proof}

The positive semi-definite property in the Proposition. \ref{remark1} prevents BT from having the lower bound in Proposition. \ref{proposition1}. More importantly, this upper bound constrains BT loss to a small range. 
Consequently, to improve the performance of BT in real applications, we design a filter that ensures $H^{v_1, v_2}$ be positive semi-definite. A simple way is to let $(g(L^{v_1})X)^{\top} g(L^{v_2}) X\approx X^{\top} X $, which is positive semi-definite due to the nature of the inner product of matrices. 
For notation simplicity, we use $K \in {\cal{R}}^{n\times n}$ to represent the filter $g(\cdot)$. The filter can be approximately obtained by solving the following problem:
\begin{equation}\label{eq: self-expression}
\min _{K}\left\|X- KX\right\|_F^2,
\end{equation}
where $\|\cdot\|_F$ denotes the Frobenius norm. Eq. \ref{eq: self-expression} could produce a trivial solution and it neglects the rich topology information in graphs. Hence, we design a regularizer to preserve the most important low-frequency information. Considering the variation in each graph, we learn a filter for each view. 
\begin{equation}\label{filter: equation}
\min _{K^v}\left\|{X} - K^v {X}\right\|_F^2+\gamma\|K^v -\varphi(L^v)\|_F^2,
\end{equation}
where $\gamma>0$ is trade-off parameter and $\varphi(L^v)$ is a typical low-pass filter \cite{zhang2019attributed}:
\begin{equation}
    \varphi(L^v) = (I - \frac{L^v}{2})^k,
\end{equation}
where $k$ is the filter order. Eq. \ref{filter: equation} can be easily solved by setting its first-order derivative w.r.t. $K^v$ to zero, which yields,
\begin{equation}\label{filter: solution1}
K^v = \left(X {X}^{\top}+\gamma I\right)^{-1}\left(\gamma \varphi(L^v)+X {X}^{\top}\right).
\end{equation}
However, the $\mathcal{O}\left(n^3\right)$ computational complexity of the inverse operation is expensive. This problem can be alleviated by using the Woodbury matrix identity \cite{higham2002accuracy}. 
Then, the graph filter $K^v$ can be reformulated as follows:
\begin{equation}\label{filter: solution2}
\begin{aligned}
K^v =\; & [\frac{1}{\gamma} I-\frac{1}{\gamma^2} X(I+\frac{1}{\gamma} {X}^{\top} {X})^{-1} {X}^{\top}]\left(\gamma \varphi(L^v) +{X} {X}^{\top}\right) \\
 =\; & \frac{X {X}^{\top}}{\gamma} -\frac{X (I+\frac{1}{\gamma} {X}^{\top} {X})^{-1}{X}^{\top}
\left(\gamma \varphi(L^v)+X{X}^{\top}\right)}{\gamma^2}\\
 \; & + \varphi(L^v).
\end{aligned}
\end{equation}

The complexity is reduced from $\mathcal{O}\left(n^3\right)$ to $\mathcal{O}\left(n^2 f\right)$ when $n>f$. It can be seen from Eq. \ref{filter: solution2} that the filter is essentially a feature-based modification to the conventional filter, which is solely based on graph structure. Different from existing works \cite{pan2021multi, lin2021graph}, the filters are not independent for each graph but correlated by feature, thus they can capture the correlations between views to some extent.

\subsection{Multi-relational Graph Clustering}

The multi-relational graph is first processed by the filter to obtain $\widetilde{X}^v$:
\begin{equation}
    \widetilde{X}^v = K^v X.
\end{equation}
The smoothed features $\widetilde{X}^v$ integrate view-specific topology with the node attributes. We then feed $\widetilde{X}$ into a network that consists of linear layers. Specifically, the encoder and decoder for each view share the same parameters, so the entire network can be considered as one auto-encoder.

The encoder maps $\widetilde{X}^{v}$ into feature subspace to obtain embedding $Z^v$. The embeddings from different views are paired to compute Barlow Twins, and the resulting values are averaged. We denote the result as a feature decorrelation loss:
\begin{equation}
\mathcal{L}_{FD} = \frac{1}{\tbinom{V}{2}}\sum^{V}_{v_1 \not= v_2}\sum^{V}_{v_2=1} \mathcal{L}^{v_1,v_2}_{BT}.
\end{equation}

We use target distribution and soft cluster assignment probabilities distribution to enhance clustering quality \cite{tu2021deep}. All embeddings are concatenated into $Z = [Z^1, Z^2,..., Z^V]\in R^{n\times Vd}$ as the input to this self-supervised clustering module.
The soft assignment distribution $Q$ can be formulated as:
\begin{equation}\label{soft assignment}
    q_{i j}=\frac{{\left(1+\left\|z_i-\sigma_j\right\|^2\right)}^{-1}}{\sum_{j^{\prime}}\left(1+\left\|z_i-\sigma_{j^{\prime}}\right\|^2\right)^{-1}},
\end{equation}
where $q_{i j}$ is measured by Student's $t$-distribution to indicate the similarity between node $i$'s embedding $z_i$ and clustering center $\sigma_j$ that is initialized by the centers resulting from the $k$-means implemented on $Z$.
The target distribution $P$ is computed as:
\begin{equation}\label{target distribution}
    p_{i j}=\frac{q_{i j}^2 / \sum_i q_{i j}}{\sum_{j^{\prime}}\left(q_{i j^{\prime}}^2 / \sum_i q_{i j^{\prime}}\right)}. 
\end{equation}
We then minimize the KL divergence between the $Q$ and $P$ distributions to encourage the data representation to align with cluster centers and enhance cluster cohesion \cite{kullback1951information}:
\begin{equation}\label{loss: clu}
    \mathcal{L}_{C L U}=K L(P \| Q)=\sum_i \sum_u p_{i u} \log \frac{p_{i u}}{q_{i u}}.
\end{equation}

Afterward, the decoder is utilized to obtain reconstructed features: 
\begin{equation}
    \bar{X}^v = Z^v W_{de},
\end{equation}
where $W_{de} \in R^{d \times f}$. It is worth noting that the features of certain ``easy samples" only exhibit tiny changes during reconstruction, suggesting that these nodes provide little informative input for our network. The Scaled Cosine Error (SCE) \cite{hou2022graphmae} is adopted as the reconstruction objective function in our model for down-weighting the contribution of those easy samples:
\begin{equation}\label{loss: sce}
\mathcal{L}^v_{SCE} = \sum_{i=1}^n \left(1 - \frac{{(\widetilde{X}^v_i)^{\top} \bar{X}^v_i}}{{\|\widetilde{X}^v_i\| \cdot \|\bar{X}^v_i\|}}\right)^2
\end{equation}
Same as Barlow Twins,  $\mathcal{L}_{SCE}$ needs to be summed and averaged:
\begin{equation}
\mathcal{L}_{MSCE} = \frac{1}{V} \sum^{V}_{v=1} \mathcal{L}^{v}_{SCE}.
\end{equation}
By combining $\mathcal{L}_{FD}$, $\mathcal{L}_{MSCE}$, and $\mathcal{L}_{C L U}$, the overall objective function of BTGF can be computed as:
\begin{equation}\label{loss: all}
    \mathcal{L}=\mathcal{L}_{MSCE} + \mathcal{L}_{FD}+ \mathcal{L}_{CLU}.
\end{equation}
We minimize the above objective function to optimize our auto-encoder and achieve the cluster label $y_i$ for node $i$ by:
\begin{equation}\label{node clustering}
    y_i=\underset{j}{\operatorname{argmax}}\hspace{0.4mm} q_{i j}.
\end{equation}

\section{Experiment}

\begin{table}[httt]
\caption{The statistics of datasets.}
\label{table:dataset}
 \resizebox{\linewidth}{!}{
\begin{tabular}{cccccc}
\toprule[2pt]%
\textbf{Datasets} & \textbf{Nodes} & \textbf{Relation Types}                                                                                      & \textbf{Edges}                                                       & \textbf{Attributes} & \textbf{Classes} \\ \midrule[1pt]%
\textbf{ACM}      & 3,025          & \begin{tabular}[c]{@{}c@{}}Paper-Author-Paper\\ Paper-Subject-Paper\end{tabular}                             & \begin{tabular}[c]{@{}c@{}}29,281\\ 2,210,761\end{tabular}           & 1, 830              & 3                \\ \midrule[1pt]%
\textbf{AMiner}   & 6,564          & \begin{tabular}[c]{@{}c@{}}Paper-Author-Paper\\ Paper-Reference-Paper\end{tabular}                           & \begin{tabular}[c]{@{}c@{}}15,412\\ 123,260\end{tabular}             & 6,564               & 4                \\ \midrule[1pt]%
\textbf{DBLP}     & 7, 907         & \begin{tabular}[c]{@{}c@{}}Paper-Paper-Paper\\ Paper-Author-paper\end{tabular}                               & \begin{tabular}[c]{@{}c@{}}90,145\\ 144,783\end{tabular}             & 2, 000              & 4                \\ \midrule[1pt]%
\textbf{Amazon}   & 7, 621         & \begin{tabular}[c]{@{}c@{}}Item-AlsoView-Item\\ Item-AlsoBought-Item\\ Item-BoughtTogether-Item\end{tabular} & \begin{tabular}[c]{@{}c@{}}266,237\\ 1,104,257\\ 16,305\end{tabular} & 2, 000              & 4                \\ \bottomrule[2pt]%
\end{tabular}}
\end{table}

\begin{table*}[!htb]
\caption{Node clustering results.}
\label{clusering results}
\resizebox{\linewidth}{!}{
\begin{tabular}{c|cccc|cccc|cccc|cccc}
\toprule[2pt]%
\textbf{Dataset}   & \multicolumn{4}{c|}{\textbf{ACM}}                                                                                       & \multicolumn{4}{c|}{\textbf{DBLP}}                                                                                                         & \multicolumn{4}{c|}{\textbf{Amazon}}                                                                                                   & \multicolumn{4}{c}{\textbf{AMiner}}                                                                                           \\ \midrule[1pt]%
\textbf{Metric}    & ACC                                    & F1                            & NMI            & ARI                           & ACC                                    & F1                                     & NMI             & ARI                                    & ACC                                    & F1                            & NMI                           & ARI                           & ACC                           & F1                            & NMI                           & ARI                           \\ \midrule[1pt]%
HAN(2019)          & 0.8823                                 & 0.8844                        & 0.5881         & 0.5933                        & 0.7651                                 & 0.6309                                 & 0.4866          & 0.4635                                 & 0.4355                                 & 0.4246                        & 0.1120                        & 0.0362                        & {0.7119} & {0.5340} & {0.2020} & {0.1260}  \\
O2MAC(2020)        & 0.9042                                 & 0.9053                        & 0.6923         & 0.7394                        & 0.7267                                 & 0.7320                                 & 0.4066          & 0.4036                                 & 0.4428                                 & 0.4424                        & 0.1344                        & 0.0898                        & 0.4939                        & {0.3202} & {0.0857} & {0.0552} \\
MvAGC(2021)        & 0.8975                                 & 0.8986                        & 0.6735         & 0.7212                        & 0.7221                                 & 0.7332                                 & 0.4191          & 0.4049                                 & 0.5188                                 & 0.5072                        & 0.2322                        & 0.1141                        & {0.5472} & {0.1781} & 0.0452                        & {0.0003} \\
HDMI(2021)         & 0.8740                                 & 0.8720                        & 0.6450         & 0.6740                        & 0.8010                                 & {0.7898}          & 0.5820          & {0.5356}          & {0.5251}          & {0.5448} & {0.3702} & {0.2735} & 0.4032                        & {0.3023} & {0.1349} & {0.0314} \\
MCGC(2021)         & 0.9147                                 & 0.9155                        & 0.7126         & 0.7627                        & 0.7850                                 & {0.7359}          & 0.5510          & {0.4439}          & {0.4683}          & {0.4804} & {0.2149} & {0.1056} & {0.4165} & {0.3982} & {0.2254} & {0.1608} \\
MGCCN(2022)        & 0.9167                                 & 0.8472                        & 0.7095         & 0.7688                        & 0.8301                                 & 0.7336                                 & 0.6156          & 0.5876                                 & 0.5309                                 & 0.4572                        & 0.1931                        & 0.1860                        & 0.6039                        & 0.5311                        & 0.2039                        & 0.2883                        \\
MGDCR(2023)        & 0.9190                                 & {0.8678} & 0.7210         & {0.6496} & 0.8070                                 & {0.8048}          & 0.6140          & {0.5259}          & {0.3489}          & {0.2039} & {0.0318} & {0.0055} & {0.5150} & {0.2533} & {0.0265} & {0.0300} \\ \midrule[1pt]%
\textbf{BTGF} & {\textbf{0.9322}} & \textbf{0.9331}               & \textbf{0.758} & \textbf{0.8085}               & {\textbf{0.8309}} & {\textbf{0.8384}} & \textbf{0.6242} & {\textbf{0.5969}} & {\textbf{0.6603}} & \textbf{0.6612}               & \textbf{0.3853}               & \textbf{0.2829}               & \textbf{0.7308}               & \textbf{0.5408}               & \textbf{0.3603}               & \textbf{0.5233}                \\ \bottomrule[2pt]%
\end{tabular}}
\end{table*}

\subsection{Datasets and Metrics}
To show the effectiveness of BTGF, we evaluate our method on four multi-relational graphs. ACM and DBLP \cite{fan2020one2multi} are citation graph networks. Amazon \cite{he2016ups} is a review graph network. AMiner \cite{wang2021self} is an academic graph network. The statistical information of them is summarized in Table \ref{table:dataset}.
We adopt four popular clustering metrics, including ACCuracy (ACC), Normalized Mutual Information (NMI), F1 score, and Adjusted Rand Index (ARI). A higher value of them indicates a better performance.

\subsection{Experimental Setup}
{
\setlength{\parindent}{0cm}
\paragraph{\bf{Parameter Setting}}The experiments are implemented in the PyTorch platform
using an Intel(R) Core(TM) i9-12900k CPU, and GeForce GTX 3090 24G GPU.
Our auto-encoder is trained by Adam optimizer \cite{kingma2017adam} for 400 epochs. 
For simplicity, both our encoder and decoder only have one linear layer, denoted as $W \in R^{f \times d}$, $W_{de} \in R^{d\times f}$ respectively, where $d=10$.
The learning rate and weight decay of the optimizer are set to $1e^{-2}$ and $1e^{-3}$, respectively. The filter's parameters $k$ and $\gamma$ is tuned in $[1,2,3,4,5]$ and $[0.1,1,10,100,1000]$, respectively.
% Moreover, we set $\alpha=1$ in clustering loss.
% and the balanced parameter $\lambda = 1e-3$ are set in loss functions.
}
{
\setlength{\parindent}{0cm}
\paragraph{\bf{Comparison Methods}}We compare BTGF with multi-view methods: HAN \cite{10.1145/3308558.3313562} and HDMI \cite{jing2021hdmi}, which use the attention mechanism; O2MAC \cite{fan2020one2multi} clusters multi-relational data by selecting an informative graph. We also compare BTGF with contrastive learning methods, such as MCGC \cite{pan2021multi} and MGCCN \cite{10.1016/j.ins.2022.09.042}. MvAGC \cite{lin2021graph} and MGDCR \cite{mo2023multiplex} are also compared. Particularly, MvAGC and MCGC both employ graph filters to preprocess the attribute. MGDCR utilizes Barlow Twins as an objective function for multi-relational clustering. For an unbiased comparison, we copy part of the results from MCGC.
}

\subsection{Clustering Results}

The clustering results are reported in Table \ref{clusering results}. From it, we have the following observations:
\begin{itemize}
    \item The advantages of BTGF are clear when compared to deep multi-view clustering methods: HAN, HDMI, O2MAC, and MGCCN. With respect to the most recent MGCCN, our method improves ACC, F1, NMI, ARI by 11\%, 16\%, 45\%, 34\% on average, respectively. Note that MGCCN uses a contrastive learning mechanism to fuse representations of different views. %Both HAN and HDMI apply the attention mechanism to fuse different relational information. O2MAC extracts shared representations of multiple views by leveraging the most informative graph and attribute data. 
    As shown in Fig. \ref{collapse}, our superiority stems from the incorporation of Barlow Twins, which significantly enhances the integration of information from different graphs.

    \item In comparison to shallow MvAGC and MCGC methods, our approach greatly boosts clustering performance. Both MvAGC and MCGC employ a low-pass filter on each graph to smooth the attributes, which fails to consider the correlation between different views. This is exactly the issue that BTGF aims to tackle. %And our filter employs a fixed parameter $\gamma$ across all layers of graphs, eliminating the need for fine-tuning.
    
    \item BTGF outperforms Barlow Twins-based method MGDCR. In particular, the improvement on Amazon and AMiner is significant. This could be caused by the large number of hyperparameters introduced in computing the pairwise loss in MGDCR, which are hard to find the most appropriate values.% manually fine-tuned.
\end{itemize}
In summary, BTGF consistently outperforms all compared methods in terms of four metrics over all datasets. The stable results obtained with such a simple network demonstrate the validity of our filter and the clustering architecture.
They could be further improved if a more complex auto-encoder architecture with deep artifices such as activation functions or dropout is applied. 
\begin{table}[ht]
\centering
\caption{Clustering performance of variants of BTGF.}
\label{ablation: loss}
 \resizebox{\linewidth}{!}{
\begin{tabular}{cccccc}
\toprule[2pt]%
\multicolumn{2}{c}{\textbf{Datasets}}              & \textbf{ACM}    & \textbf{DBLP}   & \textbf{Amazon} & \textbf{AMiner} \\ \midrule[1pt]%
\multirow{3}{*}{\textbf{ACC}} & \textbf{BTGF} & \textbf{0.9322} & \textbf{0.8309} & \textbf{0.6603} & \textbf{0.7308} \\
                              & w/o $\mathcal{L}_{FD}$             & 0.9121          & 0.8212          & 0.5885          & 0.6817          \\
                              & w/o $\mathcal{L}_{MSCE}$            & 0.9296          & 0.8041          & 0.5149          & 0.5506          \\ \midrule[1pt]%
\multirow{3}{*}{\textbf{F1}}  & \textbf{BTGF} & \textbf{0.9331} & \textbf{0.8384} & \textbf{0.6612} & \textbf{0.5308} \\
                              & w/o $\mathcal{L}_{FD}$             & 0.9130          & 0.8273          & 0.5272          & 0.4688          \\
                              & w/o $\mathcal{L}_{MSCE}$            & 0.9305          & 0.8080          & 0.4564          & 0.3209          \\ \midrule[1pt]%
\multirow{3}{*}{\textbf{NMI}} & \textbf{BTGF} & \textbf{0.758}  & \textbf{0.6242} & \textbf{0.3853} & \textbf{0.3603} \\
                              & w/o $\mathcal{L}_{FD}$             & 0.6988          & 0.5900          & 0.3695          & 0.3540          \\
                              & w/o $\mathcal{L}_{MSCE}$            & 0.7525          & 0.5760          & 0.2223          & 0.0612          \\ \midrule[1pt]%
\multirow{3}{*}{\textbf{ARI}} & \textbf{BTGF} & \textbf{0.8085} & \textbf{0.5969} & \textbf{0.2829} & \textbf{0.5233} \\
                              & w/o $\mathcal{L}_{FD}$             & 0.7553          & 0.5788          & 0.2590          & 0.5017          \\
                              & w/o $\mathcal{L}_{MSCE}$            & 0.8019          & 0.5501          & 0.1640          & 0.1054          \\ \bottomrule[2pt]%
\end{tabular}}
\end{table}
\subsection{Ablation Study}
\subsubsection{The Effect of Different Loss Terms}

To validate the effectiveness of different components in our model, we compare the performance of BTGF with its two variants:
\begin{itemize}
    \item Employing BTGF without $\mathcal{L}_{FD}$ to show the importance of utilizing feature decorrelation. %Only $\mathcal{L}_{CLU}$ and $\mathcal{L}_{MSCE}$ are calculated.
    \item Employing BTGF without $\mathcal{L}_{MSCE}$ to observe the impact of the decoder and feature reconstruction. It means training an encoder with $\mathcal{L}_{FD} + \mathcal{L}_{CLU}$ and no decoder. 
\end{itemize}
Based on Table \ref{ablation: loss}, we can draw the following conclusions. Firstly, the results of BTGF are better than all variants, which indicates the validity of each term. Second, the Barlow Twins plays a crucial role in fusing different views. Especially for Amazon with three relation types, the impact of the correlation regularizer $\mathcal{L}_{FD}$ is huge. In addition, the feature reconstruction loss used to down-weight easy samples’ contribution is proved to be helpful in improving the clustering performance.

\begin{figure*}[htb]
\centering
\subfigure[ACM]{
\includegraphics[width=4.6cm]{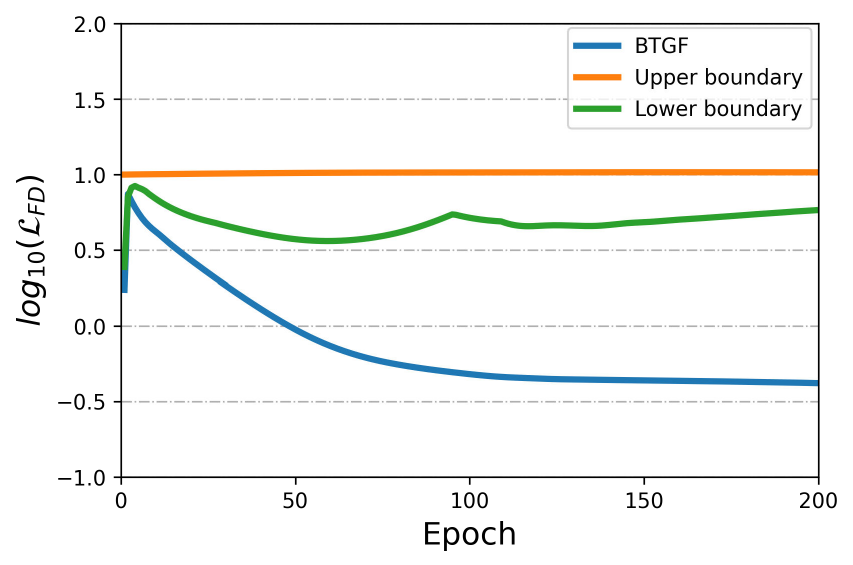}
\hspace{-5mm}
}
\subfigure[DBLP]{
\includegraphics[width=4.6cm]{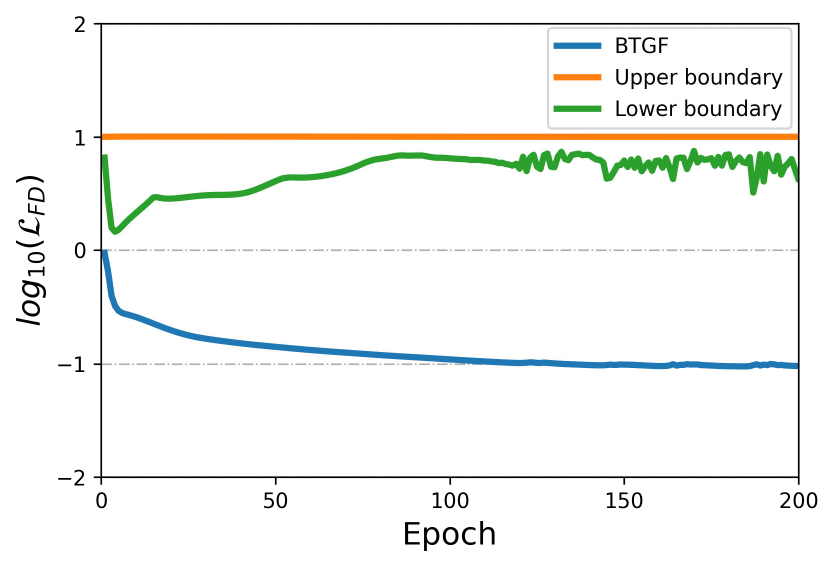}
\hspace{-5mm}
}
\subfigure[Amazon]{
\includegraphics[width=4.6cm]{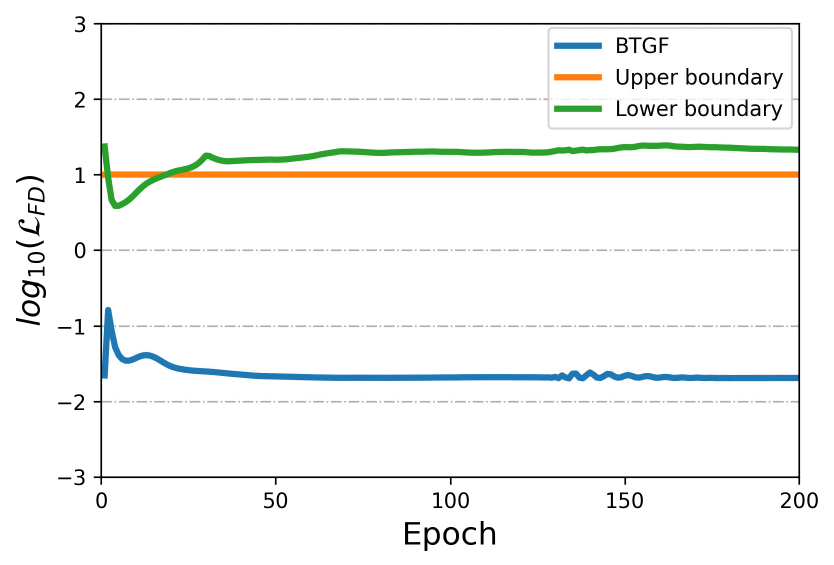}
\hspace{-5mm}
}
\subfigure[AMiner]{
\includegraphics[width=4.6cm]{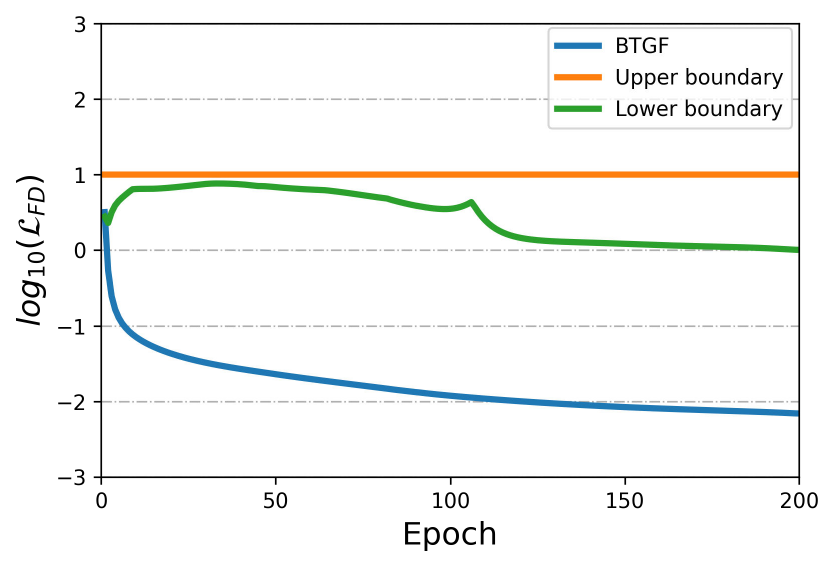}
\hspace{-5mm}
}
\caption{Verification of $\mathcal{L}_{FD}$ surpassing the lower boundary as well as having an upper boundary.}
\label{curves: bound}
\end{figure*}

\begin{figure*}[htb]
\centering
\subfigure[ACM]{
\includegraphics[width=4.6cm]{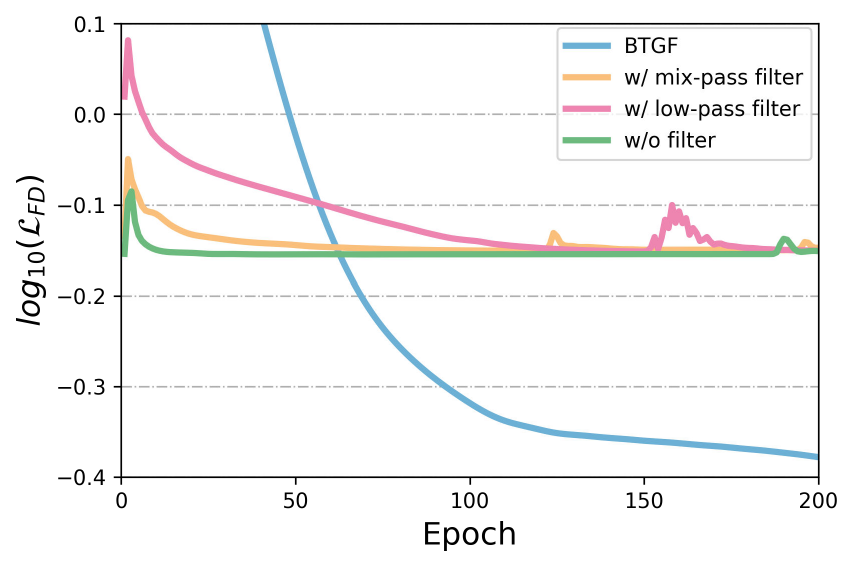}
\hspace{-5mm}
}
\subfigure[DBLP]{
\includegraphics[width=4.6cm]{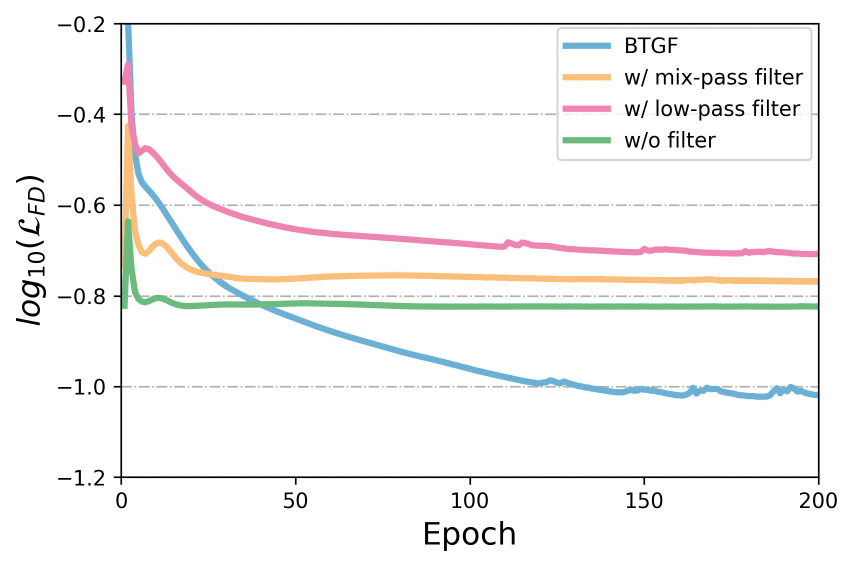}
\hspace{-5mm}
}
\subfigure[Amazon]{
\includegraphics[width=4.6cm]{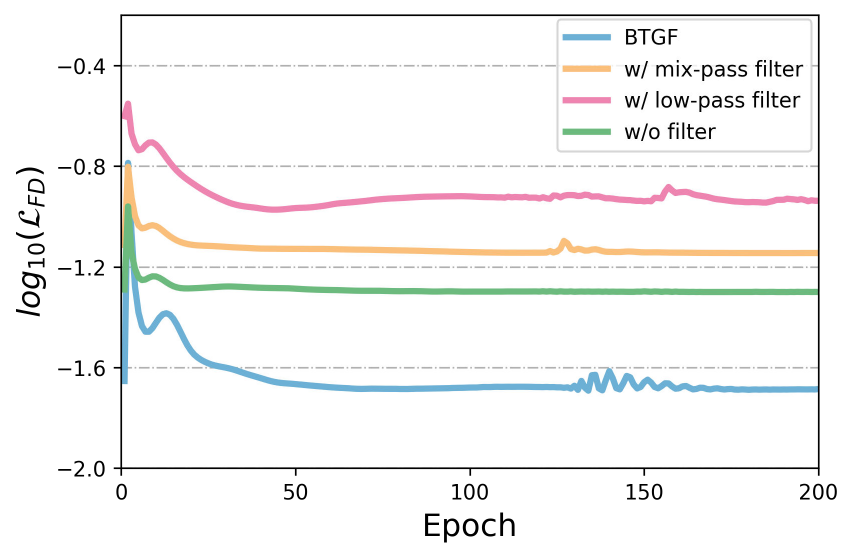}
\hspace{-5mm}
}
\subfigure[AMiner]{
\includegraphics[width=4.6cm]{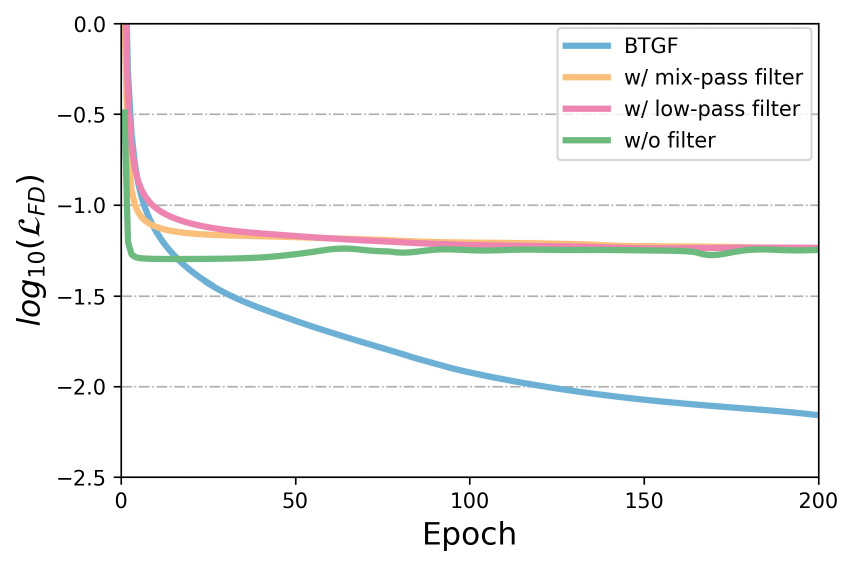}
\hspace{-5mm}
}
\caption{The evolution of feature decorrelation loss $\mathcal{L}_{FD}$ when using different filters.}
\label{curves: barlowtwins}
\end{figure*}

\begin{table}[ht]
\centering
\caption{Clustering accuracy with different filters.}
\label{ablation: filter}
 \resizebox{\linewidth}{!}{
\begin{tabular}{ccccc}
\toprule[2pt]%
\textbf{Datasets}    & \textbf{ACM}    & \textbf{DBLP}   & \textbf{Amazon}               & \textbf{AMiner} \\ \midrule[1pt]%
\textbf{BTGF}   & \textbf{0.9322} & \textbf{0.8309} & \textbf{0.6603}               & \textbf{0.7308} \\ \midrule
w/ mix-pass filter      & 0.9250          & 0.8148          & 0.6462                        & 0.5716          \\ \midrule[1pt]%
w/ low-pass filter      & 0.9088          & 0.8195          & 0.6184                        & 0.7116          \\ \midrule[1pt]%
% w/ high-pass filter     & 0.3904          & 0.4010          & 0.4316                        & 0.5399          \\ \midrule
w/o filter & 0.8707          & 0.7166          & 0.6004     & 0.3326          \\ \bottomrule[2pt]%
\end{tabular}}
\end{table}

\begin{figure*}[!htbp]
\centering
\subfigure[ACM]{
\includegraphics[width=4.5cm]{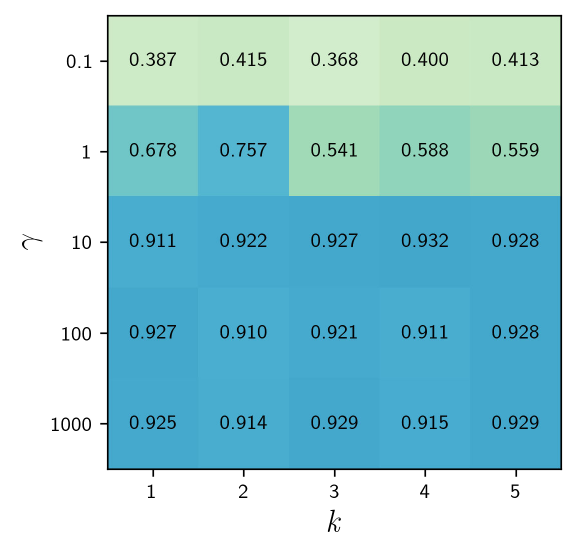}
\hspace{-5mm}
%\caption{fig1}
}
% \quad
\subfigure[DBLP]{
\includegraphics[width=4.5cm]{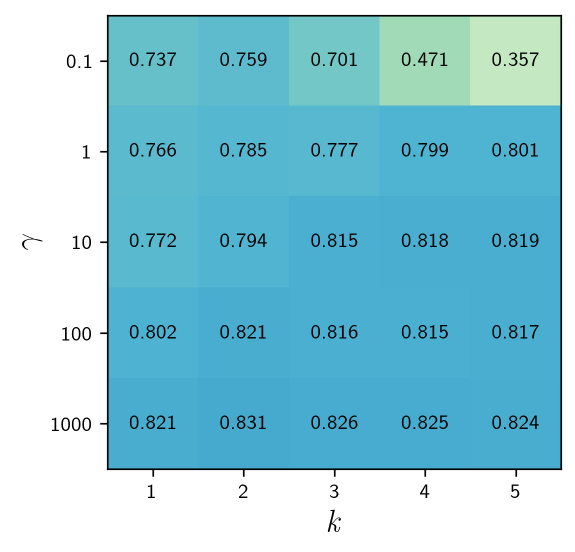}
\hspace{-5mm}
}
% \quad
\subfigure[Amazon]{
\includegraphics[width=4.5cm]{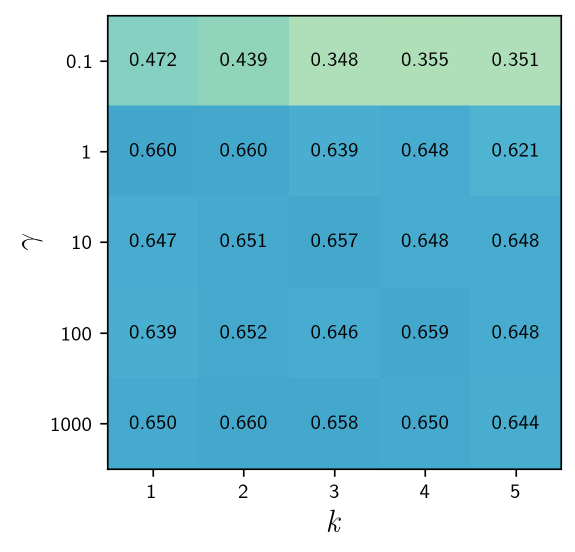}
\hspace{-5mm}
}
% \quad
\subfigure[AMiner]{
\includegraphics[width=4.5cm]{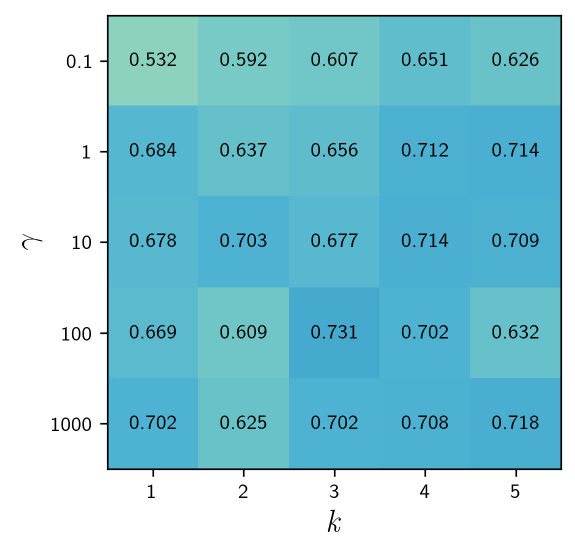}
% \hspace{-5mm}
}
\caption{The influence of $k$ and $\gamma$ on four datasets (metric: ACC).}
\label{param: filter}
\end{figure*}
\subsubsection{The Effect of Graph Filter}
To illustrate the superiority of our filter, we plot the feature decorrelation loss $\mathcal{L}_{FD}$ with respect to the training epochs and examine three variants of BTGF:
\begin{itemize}
    \item with low-pass filter $K = (I-\frac{L}{2})^2 $ instead of Eq. \ref{filter: solution2};
    \item with mix-pass filter $K = (I-\frac{L}{2})^2 + (\frac{L}{2})^2$ instead of Eq. \ref{filter: solution2};
    \item without using graph filter as preprocessing.
\end{itemize}
There are three findings:\\
\\
\textbf{Obs 1:} As shown in Fig. \ref{curves: bound}, our $\mathcal{L}_{FD}$ loss has the smallest value. According to our Propositions, the positive semi-definite constraint on the feature can restrict the upper boundary of $\mathcal{L}_{FD}$ and break through the lower boundary, reaching a lower value. The evolution of $\mathcal{L}_{FD}$ is in line with our expectations. Note that the upper and lower bounds depend on the output of the encoder, exhibiting variation among different methods and datasets, and evolving throughout the training process. It is interesting to see that our upper bound is even smaller than the lower bound on Amazon, suggesting that we can find a better solution than an input with negative semi-definite inner product. This could explain our dramatic improvement on Amazon.\\
\\
\textbf{Obs 2:} As shown in Fig. \ref{curves: barlowtwins}, our $\mathcal{L}_{FD}$ converges to smaller values compared to other filters. This indicates that our filter is more effective, which in turn improves the performance of downstream tasks. This is demonstrated in Table \ref{ablation: filter}. \\ 
\\
\textbf{Obs 3:} The absence of filters produces the worst results, which validates the importance of fusing the topology structure and attribute information. Moreover, the low-pass and mix-pass filters generate inferior performance due to their neglect of the correlation between the multi-relational graphs.

\subsection{Parameter Analysis}

This section analyzes the sensitivity of parameters in BTGF across four datasets. There are only two parameters, $k$ and $\gamma$, in the graph filter. Their influence on precision is demonstrated in Fig. \ref{param: filter}. Overall, BTGF produces reasonable performance. On the one hand, a smaller value of $\gamma$ prevents the filter from effectively utilizing the topological information of multi-relational graphs, leading to impaired clusters. On the other hand, an excessively large $\gamma$ could cause the filter to completely converge to a low-pass filter, which in turn neglects the constraint on the input. In addition, BTGF performs well for a small range of $k$, which is also convenient for practical application. %In fact, too large values of $k$ may introduce the feature information of high-order neighbors as noise that deteriorates performance.

\section{Conclusion}
In this work, we find that the graph filter can impact the performance of Barlow Twins by analyzing the conditions for the existence of lower and upper bounds of Barlow Twins loss. We prove that an input with a positive semi-definite inner product yields an upper bound, which is potentially beneficial when applying Barlow Twins. Based on this finding, we design a novel graph filter that captures the correlations between multi-relational graphs. Afterward, a simple architecture, which incorporates multi-view correlation, feature decorrelation, and graph filtering, is developed for multi-relational clustering. Comprehensive experiments on four benchmark datasets demonstrate the effectiveness of our approach. The filter design guided by loss warrants further investigation in other scenarios, e.g., GNN.

\section{Acknowledgments}
This work was supported by the National Natural Science
Foundation of China (Nos. 62276053 and 62376055).

\bibliography{aaai24}

\end{document}